\documentclass[sigconf]{acmart}

\usepackage{booktabs} 

\usepackage{tikz}
\usepackage{todonotes}

\usepackage{ifthen}
\usepackage{environ}
\usepackage{xspace,url}
\usepackage{hyperref}

\usepackage{tikz}
\usetikzlibrary{shapes,decorations,shadows,matrix,arrows}

\tikzstyle{arg}=[draw, thick, circle, fill=gray!15,inner sep=3pt]

\newcommand{\nop}[1]{}

\newcommand{\parents}[1]{\ensuremath{\mathit{par}(#1)}\xspace}

\def\meet{\sqcap}

\newcommand{\tvt}{\ensuremath{\mathbf{t}}\xspace}
\newcommand{\tvf}{\ensuremath{\mathbf{f}}\hspace*{0.005mm}\xspace}
\newcommand{\tvu}{\ensuremath{\mathbf{u}}\xspace}

\newcommand{\cf}{\textit{cf}}
\newcommand{\adm}{\textit{adm}}
\newcommand{\prf}{\textit{prf}}
\newcommand{\pref}{\prf}

\newcommand{\grd}{\textit{grd}}

\acmArticle{4}

\copyrightyear{2021}
\acmYear{2021}
\setcopyright{acmlicensed}\acmConference[SAC '21]{The 36th ACM/SIGAPP
	Symposium on Applied Computing}{March 22--26, 2021}{Virtual Event,
	Republic of Korea}
\acmBooktitle{The 36th  ACM/ SIGAPP Symposium on Applied Computing (SAC
	'21), March 22--26, 2021, Virtual Event, Republic of Korea}
\acmPrice{15.00}
\acmDOI{10.1145/3412841.3441962}
\acmISBN{978-1-4503-8104-8/21/03}

\begin{document}
	
	\title{Strong Admissibility for Abstract Dialectical Frameworks}
	\author{Atefeh Keshavarzi Zafarghandi}
	\affiliation{%
		\institution{Department of Artificial Intelligence, Bernoulli Institute \\ University of Groningen}
		\city{Groningen} 
		\state{ The Netherlands} 
	}
	\email{A.Keshavarzi.Zafarghandi@rug.nl}
	
	\author{Rineke Verbrugge}
	\affiliation{%
		\institution{Department of Artificial Intelligence, Bernoulli Institute \\ University of Groningen}
		\city{Groningen} 
		\state{The Netherlands} 
	}
	\email{L.C.Verbrugge@rug.nl}
	
	\author{Bart Verheij}
	\affiliation{%
		\institution{Department of Artificial Intelligence, Bernoulli Institute \\ University of Groningen}
		\city{Groningen} 
		\state{The Netherlands} 
	}
	\email{Bart.Verheij@rug.nl}
	\renewcommand{\shortauthors}{A. Keshavarzi Zafarghandi et al.}
	\begin{abstract}
		Abstract dialectical frameworks (ADFs) have been introduced as a formalism for  modeling and evaluating  argumentation allowing general logical satisfaction conditions. Different criteria used to settle the acceptance of arguments are called semantics. Semantics of ADFs have so far mainly been defined based on the concept of admissibility. However, the notion of strongly admissible semantics studied for abstract argumentation frameworks has not yet been introduced for ADFs. In the current work we present the concept of strong admissibility of interpretations for ADFs. Further, we show that strongly admissible interpretations of ADFs form a lattice  with the grounded interpretation as top element. 
	\end{abstract}

	
	\maketitle
	\section{Introduction}
	Interest and attention in artificial intelligence-related areas in argumentation theory has been increasing, by the wide variety of formalisms of argumentation to model argumentation and by the variety of semantics that clarify the acceptance of arguments \cite{baroni2018handbook,DBLP:books/sp/HAT2014}.
	Abstract argumentation frameworks (AFs) as introduced by Dung \cite{Dung95onthe} are a core formalism in formal argumentation, (have proven successful in many applications related to multi-agent systems \cite{DBLP:conf/argmas/2011}). 
	Abstract dialectical frameworks (ADFs) were first introduced in \cite{DBLP:conf/kr/BrewkaW10},  further refined in~\cite{DBLP:conf/ijcai/BrewkaSEWW13,DBLP:journals/flap/BrewkaESWW17}. They are expressive generalizations of AFs in which the logical relations among arguments can be represented.

	A key question in formal argumentation is  `How is it possible to evaluate 
	arguments 
	in a given  formalism$?$'  Answering  this question leads to the  introduction of  several types of semantics. 
	Different semantics reflect
	different types of point of view about the acceptance or denial  of arguments. Most of the semantics of AFs/ADFs are based on the concept
	of admissibility, in \cite{DBLP:journals/ai/CaminadaA07} it is shown that   admissibility  plays  an important  role  w.r.t.  rationality postulates. 
	
	It is shown in \cite{DBLP:journals/flap/BrewkaESWW17}, that each AF can be represented as an ADF, further, it is shown that  semantics defined for ADFs are proper generalizations of the semantics of AFs. However, some of the semantics of AFs  have not yet been introduced for ADFs, namely \emph{strongly admissible semantics}. In the current work we introduce strongly admissible semantics of ADFs.
	
	In ADFs an interpretation is called \emph{admissible} if it does not contain any unjustifiable information.  An interpretation is called \emph{preferred} if it is a maximal admissible interpretation. Thus, each admissible interpretation is contained in a preferred interpretation. That is, to answer the credulous decision problem under preferred semantics it is enough to answer the problem under admissible semantics.  In addition, an
	interpretation is \emph{grounded} if it collects all the information that is beyond any doubt.
	
	In AFs the concept of strongly admissible semantics has first been defined in the work of Baroni and
	Giacomin  \cite{DBLP:journals/ai/BaroniG07}, based on the notion of strong defence. 
	Later in \cite{DBLP:conf/comma/Caminada14}   this concept was introduced without referring to strong defence.  Further,   in \cite{DBLP:journals/argcom/CaminadaD19}  Caminada and Dunne presented a labelling account of strong admissibility  to answer the credulous decision problem of AFs under grounded semantics.  	In \cite{DBLP:journals/flap/Caminada17, DBLP:conf/comma/Caminada14, DBLP:journals/argcom/CaminadaD19}   it was shown that strong admissibility plays a critical role in discussion games for AFs under grounded semantics. That is,  it has been shown that strongly admissible extensions/labellings make a lattice with the maximum element of the grounded extension of a given AF. Therefore,   the concept of strong admissibility semantics of AFs relates to grounded semantics of AFs in a similar way as the relation between admissible semantics of AFs and preferred semantics of AFs. That is, to answer the credulous decision problem of AFs under grounded semantics it is enough to solve the decision problem   for AFs under strongly admissible semantics. 
	
	In \cite{DBLP:conf/comma/KeshavarziVV20}, a discussion game was introduced to answer the credulous decision problem of ADFs under grounded semantics without constructing the full grounded interpretation of the given ADF. However, the concept of strongly admissible semantics of ADFs has not been introduced. 
	
	This was a motivation for us to present the notion of strongly admissible semantics for ADFs in this work. 
	However, studying  whether the game that is presented in \cite{DBLP:conf/comma/KeshavarziVV20} is equivalent to constructing a strongly admissible interpretation that satisfies the claim, in the given ADF, is beyond the topic of this work and is left for future research.
	
	Semantics of AFs are usually defined based on extensions using the notion of argument acceptability. In contrast, semantics of ADFs are defined in terms of three-valued interpretations using both argument acceptability and deniability. In this sense, there is a connection with the use of labelings for AFs using argument acceptability/deniability (e.g., [16]). However, by the use of general propositional formulas as argument acceptance conditions, ADFs allow for richer relations between arguments than AFs, which only allow attack. 
	
	As a result, because of the special structure of ADFs, the definition of strong admissibility semantics of AFs cannot be directly reused in ADFs. Thus, we first present the notion strong acceptability/deniability of arguments in an interpretation. Then, we present the concept of strong admissibility to characterise the properties of the grounded interpretation of ADFs.
	
	The presented notion of strong admissibility for ADFs is closely related to strong admissibility for AFs in three ways. First strong admissibility is defined in terms of strongly acceptable/deniable  arguments the truth value of which presented in a given interpretation.  Second such strongly acceptable/deniable arguments are recursively reconstructed from their strongly acceptable/deniable parents. Third there is a close relation to the grounded semantics, in the formally precise sense that the maximal element of the lattice of strongly admissible sets is the grounded interpretation.

	This paper is structured as follows.	In Section \ref{sec: bac},   we  present  the  relevant  background.  Then, in Section \ref{sec: sadf}, the main contribution of our work is to introduce the concept of strongly admissible semantics for ADFs. Then
	we show that in each ADF, the set of strongly admissible interpretations  form a lattice with the trivial interpretation as the unique minimal element and the grounded interpretation as the unique maximal element
	In Section \ref{sec: con}, we present a   conclusion of our work
	and we present some future research questions arising from this work.

	\section{Formal Preliminaries} \label{sec: bac}
	In this section,	 we only briefly present the syntax of AFs \cite{Dung95onthe}. We present the concept of strongly admissible semantics of AFs due to \cite{ DBLP:journals/ai/BaroniG07}. Then, we present ADFs due to \cite{DBLP:conf/kr/BrewkaW10,DBLP:conf/ijcai/BrewkaSEWW13, DBLP:journals/flap/BrewkaESWW17}.
	
	\subsection{Abstract Argumentation Frameworks}
	We start the preliminaries to our work by recalling the basic notion of Dung's abstract argumentation frameworks (AFs) \cite{Dung95onthe} and the concept of strong admissibility semantics of AFs due to Baroni and
	Giacomin \cite{DBLP:journals/ai/BaroniG07}.
	\begin{definition} \cite{Dung95onthe}
		An abstract argumentation framework (AF) is a pair $(A, R)$ in which $A$ is a set of
		arguments and $R \subseteq A \times A$ is a binary relation representing attacks among arguments.
	\end{definition}
	Let $F = (A, R)$ be a given AF. For each $a, b \in A$, the relation $(a, b) \in R$
	is used to represent that a is an argument attacking the argument b. An argument $a \in A$ is, on the other hand,
	defended by a set $S \subseteq A$ of arguments (alternatively, the argument is acceptable w.r.t. $S$) (in $F$)
	if for each argument $c \in A$, it holds that if $(c, a) \in R$ then there is a $s \in S$ such that $(s, c) \in R$ ($s$ is called a defender of $a$).
	\begin{example}\label{exp: back AF}
		Let $F= (\{a, b, c\}, \{(a, b), (b, c)\})$ be an AF. In $F$,  $(a, b)$ means that argument $a$ attacks $b$, and  $(b, c)$ means that $b$ attacks $c$. Here, argument $c$ is defended by set $\{a\}$ (alternatively, $c$ is acceptable with respect to $\{a\}$), since $a$ attacks the attacker of $c$, namely $b$. 
	\end{example}
	\noindent
	Different semantics of AFs present which sets of arguments in a given AF can be accepted jointly. \footnote{The interested reader in semantics of AFs can see \cite{Dung95onthe}. }
	Let $F= (A, R)$ be an AF, then $S\subseteq A$ is a \emph{conflict-free} set (extension), if there exists no $a, b\in S$ such that $(a, b)\in R$.  For instance, in Example \ref{exp: back AF}, the set $\{a, c\}$ is a conflict-free set of $F$. Further, a set of arguments is a \emph{grounded} extension of an AF if (intuitively) there is no doubt on the acceptance of the arguments in the set. Every AF has a unique grounded extension. In Example \ref{exp: back AF}, a unique grounded extension of $F$ is $\{a, c\}$. We avoid here to present the formal definition of the grounded extension However, in Example \ref{exp: back AF}, the intuition is that  $a$ is not attacked by any argument, thus no one has any doubt to accept argument $a$. Argument $c$ is attacked by $b$, however, it is defended by $a$ which was accepted by everyone. Thus, $\{a,c\}$ is a unique grounded extension of $F$. In Definition \ref{def: sadm AFs} we represent the notion of strongly admissible semantics of AFs. 
	\begin{definition}\label{def: sdefended of AFs}
		\cite{DBLP:journals/ai/BaroniG07} Given an argumentation framework, $F= (A, R)$, $a\in A$ and $S\subseteq A$, it is said that $a$ is strongly defended by $S$ if and only if  each attacker $c\in A$  of $a$ is attacked by some $s\in S\setminus\{a\}$ such that $s$ is strongly defended by $S\setminus\{a\}$.
	\end{definition}
	\noindent	In other words, $a$ is strongly defended by $S$ if for any attacker of $a$ there exists a defender $s$ for $a$ in $S$ that is not equal to $a$, i.e.\ $s\not=a$, such that $s$ is strongly defended by $S\setminus\{a\}$.
	In Example \ref{exp: back AF}, argument $c$ is strongly defended by set $S= \{a, c\}$, since the attacker of $c$, namely $b$ is attacked by $a\in S\setminus\{c\}$ and $a$ is strongly defended by $S\setminus\{c\}$. Actually, $a$ is strongly defended by $S=\emptyset$, since $a$ is not attacked by any argument.
	\begin{definition}\label{def: sadm AFs}
		Given an AF $F = (A, R)$ and set $S\subseteq A$. It is said that $S$ is a strongly admissible extension of $S$ if every $s\in S$ is strongly defended by $S$.
	\end{definition}
	
	\noindent  In Example \ref{exp: back AF}, sets $S_1= \emptyset$, $S_2= \{a\}$, and $S_3=\{a, c\}$ are strongly admissible extensions of $F$; all of them are subsets of the grounded extension of $F$. However, set $S'= \{c\}$ is not a strongly admissible extension of $F$, since $c\in S'$ is not strongly defended by $S'$. Because argument $c$ is attacked by $b$, however, no argument in $S'\setminus\{c\}$ attacks $b$.
	
	\subsection{Abstract Dialectical Frameworks}
	We briefly restate some of the key concepts of abstract dialectical frameworks that are derived from those given in~\cite{DBLP:conf/kr/BrewkaW10,DBLP:conf/ijcai/BrewkaSEWW13, DBLP:journals/flap/BrewkaESWW17}.
	\begin{definition}
		An abstract dialectical framework (ADF) is a tuple $F = (A, L, C)$
		where:
		\begin{itemize}
			\item $A$ is a finite set of arguments (statements, positions); 
			\item $L\subseteq A \times A$ is a set of links among arguments;
			\item $C = \{\varphi_a\}_{a\in A}$ is a collection of propositional formulas over arguments, called acceptance conditions.
		\end{itemize}
	\end{definition}
	
	\noindent An ADF can be represented by a graph in which nodes indicate arguments and links show the relation among arguments. Each argument $a$ in an ADF is labelled by a propositional formula, called acceptance condition,  $\varphi_a$ over $\parents a$ such that, $\parents a=\{ b\ |\ (b, a)\in L\}$.  
	The acceptance condition of each argument clarifies under which condition the argument can be accepted~\cite{DBLP:conf/kr/BrewkaW10,DBLP:conf/ijcai/BrewkaSEWW13, DBLP:journals/flap/BrewkaESWW17}. Further, acceptance conditions indicate the set of links implicitly, thus, there is no need of presenting $L$ in ADFs explicitly. 
	
	An argument $a$ is called an \emph{initial argument} if $\parents a= \{\}$. 
	An \emph{interpretation} $v$ (for $F$) is a function $v : A \mapsto \{\tvt, \tvf , \tvu\}$, that maps
	arguments to one of the three truth values
	true ($\tvt$), false ($\tvf$), or undecided ($\tvu$). Truth values can be ordered via the information ordering relation $<_i$ given by  $\tvu <_i \tvt$ and $\tvu <_i \tvf$ and no other pair of truth values are related by $<_i$. Relation $\leq_i$ is the reflexive and transitive closure of $<_i$.
	The pair $(\{\tvt, \tvf, \tvu\}, \leq_i)$ is a complete meet-semilattice with the meet operator $\sqcap_i$, such that, $\tvt\sqcap_i\tvt=\tvt$, $\tvf\sqcap_i\tvf=\tvf$, and returns $\tvu$ otherwise.
	The meet of two interpretations $v$ and $w$ is then defined as $(v\meet_i w)(a)=v(a)\meet_i w(a)$ for all $a\in A$.
	
	Further, $v$ is called \emph{trivial}, and $v$ is denoted by $v_\tvu$, if $v(a) = \tvu$ for each $a \in A$. Further, $v$ is called a two-valued interpretation if for each $a\in A$ either $v(a)=\tvt$ or $v(a)=\tvf$.
	Interpretations can be ordered via $\leq_i$ with respect to  their information content. 
	Let $\mathcal{V}$ be the set of all interpretations for an ADF $F$.
	It is said that an interpretation $v$  is an \textit{extension} of another interpretation $w$, if $w(a) \leq_i v(a)$ for each $a \in A$, denoted by $w \leq_i v$. Further, if $v\leq_i w$ and $w\leq_i v$, then $v$ and $w$ are equivalent, denoted by $v\sim_i w$. 
	
	
	For reasons of brevity, we will sometimes shorten the notion of three-valued interpretation $v= \{a_1\mapsto t_1, \dots , a_m\mapsto t_m\}$ with arguments $a_1, \dots, a_m$ and truth values $t_1, \dots , t_m$ as follows: $v= \{a_i \ |\ v(a_i)=\tvt\}\cup \{\neg a_i \ |\ v(a_i)= \tvf\}$. For instance, $v= \{a\mapsto\tvf, b\mapsto \tvt\}= \{\neg a, b\}$. We use this notation in Figure \ref{fig: lattice}. 
	
	Semantics for ADFs can be defined via the
	\emph{characteristic operator} $\Gamma_F$
	which maps interpretations to interpretations.
	Given an interpretation $v$ (for $F$), the partial valuation of $\varphi_a$ by $v$,  is
	$v(\varphi_a)= \varphi_a^v = \varphi_a[b/\top : v(b)=\tvt][b/\bot : v(b)=\tvf]$, for $b\in \parents a$.
	
	\begin{definition}
		Let $F$ be an ADF and let $v$ be an interpretation of $F$. 	Applying  $\Gamma_F$  on $v$ leads to  $ v'$  s.t. for each $a\in A$, $v'$ is as follows:
		\[
		v'(a)=\begin{cases}
			\tvt &\quad\text{if $\varphi_a^v$} \text{ is irrefutable (i.e., $\varphi_a^v$ is a tautology) },\\
			\tvf &\quad\text{if $\varphi_a^v$}\text{ is unsatisfiable   (i.e., $\varphi_a^v$ is   a contradiction)}, \\
			\tvu &\quad\text{otherwise.}
		\end{cases}
		\]
	\end{definition}
	\noindent	Note that the operator $\Gamma_F$  is monotonic, that is, when $v\leq_i w$ for interpretations $v$ and $w$, then $\Gamma_F(v)\leq_i\Gamma_F(w)$. 
	The semantics of ADFs are defined via the characteristic operator as follows. 

	\begin{definition}
		\label{def:sem:adfs}
		Given an ADF $F$, an interpretation $v$ is:
		
		\begin{itemize}
			\item conflict-free iff  $v(s)= \tvt$ implies $\varphi_s^v$ is satisfiable and $v(s)=\tvf$ implies $\varphi_s^v$ is unsatisfiable;
			\item  admissible in $F$ iff $v \leq_i \Gamma_F(v)$; 
			\item  preferred in $F$ iff $v$ is $\leq_i$-maximal admissible;
			\item the grounded interpretation of $F$ iff $v$ is the least fixed point of $\Gamma_F$.
		\end{itemize}
	\end{definition}
	\noindent The set of all $\sigma$ interpretations for an ADF $F$ is denoted by $\sigma(F)$, where $\sigma\in\{\cf, \adm,  \grd, \pref\}$ abbreviates the different semantics in the obvious manner. 
	The notion of an argument being accepted and the symmetric notion of an argument being denied in an interpretation are as follows. 
	\begin{definition}\label{def: acc and deniable}
		\vspace{-0.cm}
		Let $F=(A, L, C)$ be an ADF and let $v$ be an interpretation of $F$.
		
		\begin{itemize}
			\item An argument $a\in A$ is called acceptable with respect to  $v$ if $\varphi_a^v$ is irrefutable.
			\item An argument $a\in A$ is called deniable with respect to $v$ if $\varphi_a^v$ is unsatisfiable.
		\end{itemize}
	\end{definition}
	\begin{example}\label{exp: back ADF}
		An example of an ADF $D=(S,L,C)$ is shown in Figure \ref{fig: s adm}. To each argument a propositional formula is associated, the acceptance condition of the argument. For instance, the acceptance condition of $c$, namely $\varphi_c: \neg b\land d$, states that $c$ can be accepted in an interpretation where   $b$ is denied and $d$ is accepted.
		In $D$ the interpretation $v=\{a\mapsto\tvu, b\mapsto \tvt, c\mapsto\tvu, d\mapsto \tvu\}$ is conflict-free. However, $v$ is not an admissible interpretation, because $\Gamma_D(v)=\{a\mapsto\tvu, b\mapsto\tvu, c\mapsto\tvu, d\mapsto\tvu\}$, that is, $v\not\leq_i\Gamma_D(v)$.  
		
		The interpretation $v_1=\{a\mapsto\tvt, b\mapsto\tvu, c\mapsto\tvf, d\mapsto \tvf\}$ on the other hand is an admissible interpretation. Since $\Gamma_D(v_1)=\{a\mapsto\tvt, b\mapsto\tvt, c\mapsto\tvf, d\mapsto \tvf\}$ and $v_1\leq_i\Gamma_D(v_1)$. 
		Further, in $D$ a unique grounded interpretation $v_2= \{a\mapsto\tvt, b\mapsto\tvt, c\mapsto\tvf, d\mapsto \tvf\}$ is a preferred interpretation of $D$. 
		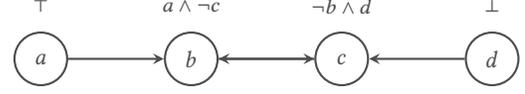
\begin{figure}[t]
			
			\centering
			\begin {tikzpicture}[black!70, >=stealth,node distance=0.6cm,
			thick,main node/.style={fill=none,draw,minimum size = 0.4cm,font=\normalsize\bfseries},
			condition/.style={fill=none,draw=none,font=\small\bfseries}]
			\path
			(0,0)node[circle, draw, minimum size = 0.7cm] (a) {$a$}
			++(2.,0) node[circle, draw, minimum size = 0.7cm] (b) {$b$}
			++(2,0) node[circle, draw, minimum size = 0.7cm] (c) {$c$}
			++(2,0) node[circle, draw, minimum size = 0.7cm] (d) {$d$};
			
			\path[thick, ->] 
			(a) edge  (b)
			(b) edge  (c)
			(c) edge  (b)
			(d) edge  (c);

			\node[condition](cs) [above of= a, yshift=0.1cm] {$\top $};
			\node[condition](cs) [above of= b, yshift=0.1cm, xshift = -0 cm] {$a\land\neg c$};
			\node[condition](cs) [above of= c, yshift=0.1cm] {$\neg b\land d$};
			\node[condition](cs) [above of= d, yshift=0.1cm] {$\bot$};
		\end{tikzpicture}
		\caption{ADF of Examples~\ref{exp: back ADF}  and \ref{exp: s adm}}
		\label{fig: s adm}
	\end{figure}
\end{example}
\noindent 
Given an ADF $F=(A, L, C)$, an argument $a\in A$ and a  semantics $\sigma\in\{\cf, \adm, \prf,  \grd\}$,   argument $a$ is \textit{credulously acceptable (deniable)} under $\sigma$ if there exists a $\sigma$ interpretation $v$ of $F$ in which $a$ is acceptable ($a$ is deniable, respectively).

In ADFs, relations between arguments  can be classified into four types, reflecting the relationship of attack and/or support that exists between the arguments.  These are listed in  Definition \ref{def: relations}. Further,  we denote the update of an interpretation $v$ with a truth value $x\in\{\tvt, \tvf, \tvu\}$ for an argument $b$ by ${v}|^{b}_{x}$, i.e. ${v}|^{b}_{x}(b) = x$ and ${v}|^{b}_{x}(a) = v(a)$ for $a\neq b$.
\begin{definition}\label{def: relations}
	Let $D=(S,L,C)$ be an ADF.
	A relation $(b,a)\in L$ is called 
	\begin{itemize}
		\item \emph{supporting} (in $D$)
		if for every two-valued interpretation $v$, $v(\varphi_a)=\tvt$ implies ${v}|^{b}_{\tvt}(\varphi_a)=\tvt$;
		\item \emph{attacking} (in $D$)
		if for every two-valued interpretation $v$, $v(\varphi_a)=\tvf$ implies ${v}|^{b}_{\tvt}(\varphi_a)=\tvf$;
		\item \emph{redundant} (in $D$) if it is both attacking and supporting;
		\item \emph{dependent} (in $D$) if it is neither attacking nor supporting.
	\end{itemize} 
	
\end{definition}

\noindent	In the current work we say that the truth value of $a$ \emph{is presented} in $v$, if $v(a)= \tvt/\tvf$.

\section{The Strongly Admissible semantics for ADFs}\label{sec: sadf}

In the following, we first  present the concept of strongly admissible semantics for ADFs. 
In ADFs, beside an argument being acceptable in an interpretation, there is a symmetric notion of an argument being deniable. Thus,   in  Definition \ref{def: a sadm w.r.t v A} we introduce the notion of strong acceptability/deniability  of an argument in an ADF with respect to a given interpretation.  In Theorem \ref{thm: sadf form a lattice}, we show that in a given ADF,  the set of  strongly admissible interpretations of $D$ make a lattice, with the unique minimal element $v_\tvu$ and the unique maximal element $\grd(D)$.

Note that in the following, $v_{|_P}$ is equal to $v(p)$ for any $p\in P$, however, it assigns  all other arguments  that do not belong to $P$ to $\tvu$. Further, in Definition \ref{def: a sadm w.r.t v A} set $S$ contains the ancestors of $a$
the truth value of which are presented in $v$, that have an effect on the truth value of $a$ in $v$. This is similar to Definition \ref{def: sdefended of AFs}, in which set $S$ contains the defenders of $a$. 
In the first item of Definition \ref{def: a sadm w.r.t v A}, set $P$ contains exactly those parents of $a$, excluding $a$, that satisfy $v(a)$ and of which the truth value is presented in  $v$. 
\begin{definition}\label{def: a sadm w.r.t v A}
	Let $D=(A, L, C)$ be an ADF and let $v$ be an interpretation of $D$. Argument $a$ is a strongly acceptable/deniable  argument with respect to interpretation $v$ and set  $S$ if the following conditions hold.  
	\begin{itemize}
		\item Let $E= \{a\}$. 	There exists a subset of parents of $a$ excluding $a$, namely $P\subseteq (\parents a\cap S)\setminus E$ such that $\varphi_a^{v_{|_P}}\equiv \top$ if $v(a)= \tvt$ and $\varphi_a^{v_{|_P}}\equiv\bot$ if $v(a)= \tvf$.
		\item Each $p\in P$, with $P$ that satisfies the first item,  is strongly acceptable/deniable with respect to interpretation $v$ and set  $S $ such that $E := E\cup \{ p\}$.
	\end{itemize}
	
\end{definition}

\noindent	Note that in Definition~\ref{def: a sadm w.r.t v A} to indicate whether an argument $a$ is strongly acceptable/deniable, we collect the set of ancestors of $a$ that affect the truth value of $a$ in   set $S$.  If the set of parents of an argument, namely $P$,   is an empty set, then $v|_P = v_\tvu$.
In Definition~\ref{def: strong adm} the concept of strong admissibility of an interpretation of a given ADF  is introduced. 

\begin{definition}\label{def: strong adm}
	Let $D= (A, L, C)$ be an ADF and let $v$ be an interpretation of $D$. An interpretation $v$ is a strongly admissible interpretation if for each $a$ such that $v(a)= \tvt/\tvf$, then $a$ is  a strongly acceptable/deniable  argument with respect to $v$ and  set $S$. 
	
\end{definition}

\noindent	These notions are clarified in Example \ref{exp: s adm}. Note that  set $S$ in Definitions \ref{def: a sadm w.r.t v A} and \ref{def: strong adm} can be the empty set. Example \ref{exp: sadf with redundant} is an instance of strong acceptability of an argument with $S= \{\}$.

\begin{example}\label{exp: s adm}
	Let $D= (\{a,b,c,d\}, \{\varphi_a: \top, \varphi_b: a\land \neg c, \varphi_c: \neg b \land d, \varphi_d: \bot \})$, depicted in Figure~\ref{fig: s adm}. Let $v= \{a\mapsto\tvu, b\mapsto \tvt, c\mapsto\tvf, d\mapsto\tvf\}$. 	We show that $c$ is strongly deniable with respect to $v$ and  set $S= \{ d\}$. To satisfy the first condition of Definition \ref{def: a sadm w.r.t v A}, we choose  the subset of parents of $c$ excluding $c$ equal to $\{d\}$. It is easy to check that $\varphi_c^{v_{|_d}}\equiv \bot$. In this step $E=\{c\}$. To check the second condition of Definition \ref{def: a sadm w.r.t v A}, we have to show that $d$ is also a strongly deniable argument. To this end, by the definition $E$ extends to $E:= E\cup \{d\}$. Further,  clearly $\varphi_d^{v_\tvu}\equiv \bot$.  Thus, $c$ is strongly deniable with respect to $v$ and set $S=\{ d\}$. In other words, set $S$ indicates a parent of $c$, namely $d$ that has affect on the truth value of $c$ in $v$. 
	
	On the other hand, $c$ is not strongly deniable with respect to $v$ and set $S= \{ b\}$. The reason is as follows. Although the first condition of Definition \ref{def: a sadm w.r.t v A} is satisfiable, that is, $\varphi_c^{v_{|_b}}\equiv\bot$,  the second condition is not satisfiable, i.e.\ $b$ is not strongly acceptable with respect to $v$. Toward a contradiction, assume  that $b$ is strongly acceptable w.r.t.  $v$. Thus, we have to choose a parent of $b$ that does not belong to $E=\{c, b\}$, namely $a$ and we have to show that $\varphi_b^{v_{|_a}}\equiv\top$.   However, $\varphi_b^{v_{|_a}}\not\equiv\top$. Therefore, $b$ is not strongly acceptable with respect to $v$.  
	
	Note that $c$ is also strongly acceptable with respect to $v$ and $S= \{c, d\}$.	In other words, $S=\{d\}$ is the least subset of $A$ that satisfies the conditions of Definition \ref{def: a sadm w.r.t v A} for $c$.

\end{example}
\noindent	Example \ref{exp: sadf with redundant} is an instance of ADFs with a redundant link.

\begin{example}\label{exp: sadf with redundant}
	Let $D= (\{a, b\},  \{\varphi_a: b\lor\neg b, \varphi_b:b\})$ be an ADF, depicted in Figure \ref{fig:  sadf with redundant}. We show that $v= \{a\mapsto\tvt, b\mapsto\tvu\}$ is a strongly admissible interpretation of $D$. To this end, we show that $a$ is strongly acceptable with respect to $v$ and $S= \{\}$. It is clear that $P\subseteq(\parents a\cap S)$ is the empty set and  $\varphi_a^{v_\tvu}$ is irrefutable. Thus, $S= \{\}$ satisfies the conditions of Definition \ref{def: a sadm w.r.t v A} for $a$. That is, $a$ is strongly acceptable with respect to $v$ and $S=\{\}$.
	\begin{figure}[h]
		
		\centering
		\begin {tikzpicture}[black!70, >=stealth,node distance=0.6cm,
		thick,main node/.style={fill=none,draw,minimum size = 0.4cm,font=\normalsize\bfseries},
		condition/.style={fill=none,draw=none,font=\small\bfseries}, scale=0.9]
		\path
		(0,0)node[circle, draw, minimum size = 0.7cm] (a) {$a$}
		++(2.,0) node[circle, draw, minimum size = 0.7cm] (b) {$b$};
		
		\path[thick, ->] 
		(b) edge  (a);
		\path [loop left,thick, distance=0.9cm, in =45, out=310, ->] (b) edge (b);

		\node[condition](cs) [above of= a, yshift=0.1cm] {$b\lor\neg b $};
		\node[condition](cs) [above of= b, yshift=0.1cm, xshift = -0 cm] {$b$};
	\end{tikzpicture}
	\caption{ADF of Examples~\ref{exp: sadf with redundant}}
	\label{fig:  sadf with redundant}
\end{figure}
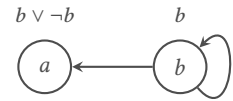

\end{example}

\noindent As we presented earlier, for instance, in Example \ref{exp: s adm}, we are interested in finding a least set $S$ of ancestors of an argument in question that satisfies the conditions of Definition \ref{def: a sadm w.r.t v A}, presented in Definition \ref{def: least set}.
\begin{definition}\label{def: least set}
Let $a$ be an argument that is strongly acceptable/\newline deniable with respect to $v $  and $S$. We say that  $S$ is \textit{a least} set that satisfies the conditions of Definition \ref{def: a sadm w.r.t v A} for $a$ if there is no $S'$ with $|S'|<|S|$  such that $a$ is strongly acceptable/deniable with respect to $v$ and $S'$.
\end{definition}
\noindent	 For instance, in Example \ref{exp: s adm}, $S=\{d\}$ is the least set that satisfies the conditions of Definition \ref{def: a sadm w.r.t v A} for $c$. We define \textit{the maximum level} of $a$ in a least set $S$ recursively, as follows. 
\begin{definition}\label{def: level based on least S}
Let $D$ be an ADF and	let $a$ be strongly acceptable/deniable with respect to $v$ and a least set $S$, and let  $P\subseteq \{\parents a \cap S\}\setminus  \{a\}$ such that  $\varphi_a^{v|_P}\equiv\top/\bot$. The \emph{maximum level} of $a$ with respect to a least set $S$ is:
\begin{itemize}
	\item  If $P= \emptyset $, then the maximum level of $a$ in $S$ is $1$.
	\item If $P\not= \emptyset $ and the maximum of  the maximum level of an argument of $P$ in $S$ is $k$, then the level of $a$ with respect to $S$ is $k+1$.
\end{itemize} 
\end{definition}
\noindent		For instance, in Example \ref{exp: s adm}, the maximum level of $c$ with respect to $S=\{d\}$ is $2$. This is because the maximum level of $d$ with respect to $S$ is $1$.

Considering ADF $D$ 
of Example \ref{exp: sadf with redundant}, by Definition \ref{def: level based on least S} the maximum level of $a$ with respect to the least set $S= \{\}$ is one. 
Lemma \ref{lem: max finite}  shows that if $a$ is strongly acceptable/deniable with respect to $v$ and $S$, then the maximum level of $a$ is finite in any given ADF.
\begin{lemma}\label{lem: max finite}
Let $D$ be an ADF, let $v$ be an interpretation of $D$ and let $a$ be an argument that is strongly acceptable/deniable with respect to $v$ and  a least set $S$. Then $a$ has a finite maximum level in in $S$.
\end{lemma}
\begin{proof}
Toward a contradiction assume that $a$ is an argument with infinite maximum level in $S$. Therefore, by Definition \ref{def: level based on least S}, the set of parents of $a$, namely $P$ with $\varphi_a^{v_{|_P}}$ is a non-empty set. Further, there exists  an argument $p$ in $P\setminus\{a\}$ with infinite  maximum level. By the same reason $p$ has a parent  with infinite maximum level that is neither equal to $a$ nor $p$. Thus, $a$ has an infinite number of ancestors. This is a contradiction by the assumption that the $D$ is a finite ADF. Thus, the assumption that $a$ has an infinite maximum level is wrong. 
\end{proof}
\begin{lemma}\label{lem: sadm in greater int}
Let $D$ be an ADF.  If $a\in A$ is strongly acceptable/
deniable with respect to  interpretation $v$ of $D$ and  a least set $S$ and $v\leq_i v'$, then $a$ is also strongly acceptable/deniable with respect to $v'$ and a least set $S$.
\end{lemma}
\begin{proof}
Since $a$ is strongly acceptable/deniable with respect to $v$ and  $S$, there exists $P\subseteq (\parents a\cap S) \setminus E$ that satisfies the first condition of Definition \ref{def: a sadm w.r.t v A}.  Since $v\leq_i v'$ the same set of parents of $a$, namely $P$ guarantees that the first condition of Definition \ref{def: a sadm w.r.t v A} holds for $a$ with respect to $v'$ and $S$. 

Assume that $S$ is also a least set that satisfies the conditions of the current lemma. We show that the second condition of  Definition \ref{def: a sadm w.r.t v A} works  by induction on the maximum level of argument $a$ in $S$.

Base case: let $a$ be an argument of the maximum  level one that is  strongly acceptable/deniable  with respect to $v$ and $S$. Therefore, $\varphi_a^{v_{\tvu}}\equiv\top/\bot$. 
Thus, $a$ is clearly strongly acceptable/deniable  with respect to  $v'$ and $S$.

Inductive step:
Assuming that this property holds for each argument of the maximum level   $j$ with $1
\leq j < i$ in $S$, i.e.,  if $a$ is an argument with the maximum level $j$ in $S$ that is strongly acceptable/deniable with respect to \ $v$ and $S$, then $a$ is strongly acceptable/deniable with respect to  $v'$ and $S$. We show that this property also holds for arguments of level $i$. Let $a$ be an argument of the maximum  level $i$. By Definition \ref{def: a sadm w.r.t v A}, there exists the set of parents of $a$, namely $P$, that satisfies the conditions of the definition with respect to  $v$ and set $S$. We claim that this $P$ also satisfies the conditions of the definition for $a$ w.r.t.  $v'$ and $S$. By Definition \ref{def: level based on least S}, the maximum level of each $p\in P$ is at most  $i-1$. Thus, by induction  hypothesis $p$ is strongly acceptable/deniable with respect to  $v'$ and set $S$. Therefore, the second condition of Definition \ref{def: a sadm w.r.t v A} also holds. Thus, $a$ is strongly acceptable/deniable  with respect to  $v'$ and $S$.
\end{proof}
\noindent A sequence of interpretations, for a given ADF $D$, is presented in Lemma \ref{lem: constructing Sadm}, each member of which is strongly admissible. In Lemma \ref{prop: big and smal sadm} it is shown that the maximum element of this sequence is the grounded interpretation of $D$.	
\begin{lemma}\label{lem: constructing Sadm}
Let $D$ be an ADF, let $v_0= v_\tvu$ and let $v_i= \Gamma_D(v_{i-1})$ for $i>0$. For each $0\leq i$ it holds that
\begin{itemize}
	\item $v_i\leq_i v_{i+1}$,
	\item $v_i$ is a strongly admissible interpretation of $D$.
\end{itemize}
\end{lemma}
\begin{proof}
\begin{itemize}
	\item The first item holds because the characteristic operator is a monotonic function.
	\item We show that each $v_i$ is a strongly admissible interpretation by induction on $i$. 
	
	Base case: For $i = 0$, it is clear that  $v_0 = v_\tvu$ is a strongly admissible interpretation. 
	
	Inductive step: Assume that $v_j$ for $j$ with $0\leq j< i$ is a strongly admissible interpretation. We show that $v_i$ is a strongly admissible interpretation. Let $a$ be an argument that is assigned to either $\tvt$ or $\tvf$ in $v_i$. If  $a\mapsto\tvt/\tvf\in v_{i-1}$, there is nothing to prove, since by the induction assumption $v_{i-1}$ is a strongly admissible interpretation. Assume that $a\mapsto\tvt\in v_i$ and $a\mapsto\tvu\in v_{i-1}$. We show that $a$ is strongly acceptable with respect to $v_i$ and  set $S$. For the case that $a\mapsto\tvf\in v_i$, the proof follows a similar method. Since $v_i(a)= \tvt$, we can conclude that $\varphi_a^{v_{i-1}}$ is irrefutable. Let $P$ be a subset of parents of $a$  the truth value of which  appears in $v_{i-1} $ 
	and $\varphi_a^{{v_{i-1}}_{|P}}\equiv\top$. 
	Otherwise, $\varphi_a^{v_{i-1}}$ cannot be irrefutable.  Thus, the first condition of Definition \ref{def: a sadm w.r.t v A} holds. 
	
	To show the second condition of Definition \ref{def: a sadm w.r.t v A}, assume that  $P\not= \{\}$. Otherwise, there is nothing to prove.  Let $p\in P$. By the induction assumption, $v_{i-1}$ is a strongly admissible interpretation. Since $p\mapsto\tvt/\tvf\in v_{i-1}$ for each $p\in P$,  $p$ is strongly acceptable/deniable with respect to  $v_{i-1}$ and  set $S$. Thus, by the  monotonicity of the characteristic operator,   $p$ is strongly acceptable/deniable with respect to $v_{i}$ and $S$. 
	Thus, the second condition of Definition \ref{def: a sadm w.r.t v A}  holds, as well.  That is, arbitrary argument $a$ is strongly acceptable with respect to $v_i$ and $S$. Thus, $v_i$ is a strongly admissible interpretation. Hence, every interpretation in the sequence $v_{\tvu}, \Gamma_D(v_\tvu), \dots$ is a strongly admissible interpretation. 
\end{itemize}
\end{proof}

\begin{lemma}\label{prop: big and smal sadm}
Let $D$ be an ADF.
\begin{itemize}
	\item $D$ has at least one strongly admissible interpretation.
	\item The least strong admissible interpretation of $D$, with respect to the $\leq_i$ ordering, is the trivial interpretation.
	\item The biggest strongly admissible interpretation, with respect to the  $\leq_i$ ordering, is the unique grounded interpretation of $D$. 
\end{itemize}
\end{lemma}

\begin{proof}
\begin{itemize}
	\item The first and the second items of the lemma are clear by Lemma \ref{lem: constructing Sadm}, which says that $v_\tvu$ is a strongly admissible interpretation.
	\item By Definition, the grounded interpretation of $D$ is the least fixed-point of the characteristic operator over $v_\tvu$ with respect to the $\leq_i$-ordering. By Lemma \ref{lem: constructing Sadm}, each $\Gamma_D^n(v_\tvu)$ is a strongly admissible interpretation. Thus, the least fixed-point of $\Gamma_D^n(v_\tvu)$  is also a strongly admissible interpretation.  Note that,  the $n$th power off $\Gamma_D$ is defined inductively, that is, $\Gamma_D^n= \Gamma_D(\Gamma_D^{n-1})$.
\end{itemize}
\end{proof}
\noindent	In Theorem \ref{thm: sadm is adm and cf} we show that each strongly admissible interpretation is  an admissible interpretation as well as conflict-free. However, the other direction of the following theorem does not work. For instance, let $D=(\{a, b\}, \{\varphi_a: \neg b\lor a, \varphi_b: \neg a\})$ be a given ADF. The interpretation $v =\{a\mapsto \tvf, b\mapsto\tvt\}$ is an admissible interpretation of $D$, however, neither $a$ nor $b$ is strongly admissible with respect to $v$. Thus, $v$ is not a strongly admissible interpretation of $D$. 
Further, $v' =\{a\mapsto \tvu, b\mapsto\tvt\} $ is a conflict-free interpretation of $D$ that is neither an admissible nor a strongly admissible interpretation. The only strongly admissible interpretation of $D$, which is also the grounded interpretation of $D$, is the trivial interpretation.
\begin{theorem}\label{thm: sadm is adm and cf}
Let $D= (A, L, C)$ be an ADF and let $v$ be a strongly admissible interpretation of $D$. Then the following hold:
\begin{itemize}
	\item $v$ is an admissible interpretation of $D$.
	\item $v$ is a conflict-free interpretation of $D$.
\end{itemize}

\end{theorem}
\begin{proof}
\begin{itemize}
	\item Let $v$ be a strongly admissible interpretation of $D$. We show that $v$ is an admissible interpretation. Toward a contradiction assume that $v$ is not an admissible interpretation, that is, $v\not\leq_i \Gamma_(D)(v)$. That is, there exists $a$ such that $v(a)= \tvt/\tvf$, but $\Gamma_D(v)(a)\not= \tvt/\tvf$. By the assumption $v$ is a strongly admissible interpretation. That is, if $v(a)=\tvt/\tvf$, then $a$ is strongly acceptable/deniable with respect to $v$ and set $S$. Thus, by the first item of Definition \ref{def: a sadm w.r.t v A}, there exists a subset of parents of $a$, namely $P$ such that $\varphi_a^{v_{|_P}}\equiv\top$ if $v(a)=\tvt$, and  $\varphi_a^{v_{|_P}}\equiv\bot$ if $v(a)=\tvf$. However, $\varphi_a^{v_{|_P}}\equiv\top$ implies that $\varphi_a^{v}$ is irrefutable and  $\varphi_a^{v_{|_P}}\equiv\top$ implies that $\varphi_a^{v}$ is unsatisfiable. The former implies if $v(a)=\tvt$, than $\Gamma_D(v)(a)=\tvt$ and the latter one implies that if $v(a)= \tvf$, then $\Gamma_D(v)(a)=\tvf$. This is a contradiction by the assumption that there exists $a$ such that $v(a)= \tvt/\tvf$, and $\Gamma_D(v)(a)\not= \tvt/\tvf$. Thus, the assumption that $v$ is not an admissible interpretation is wrong. Hence, if $v$ is a strongly admissible interpretation, then it is also an admissible interpretation.
	\item If $v$ is a strongly admissible interpretation, then by the first item of this theorem it is an admissible interpretation. By the fact that in ADFs every admissible interpretation is a conflict-free interpretation, we conclude that $v$ is a conflict-free interpretation, as well.
\end{itemize} 
\end{proof}

\subsection{The Strongly Admissible Interpretations of an ADF form a lattice }
\noindent		Although the sequence of interpretations presented in Lemma \ref{lem: constructing Sadm} produces a sequence of strongly admissible interpretations of a given ADF $D$, this sequence does not contain all of the strongly admissible interpretations of $D$.
For instance, in Example \ref{exp: s adm}, $v= \{a\mapsto \tvu, b\mapsto\tvu, c\mapsto\tvf, d\mapsto\tvf\}$ is a strongly admissible interpretation of $D$. However, $v$ is not equal to any of the elements of the sequence $v_\tvu, \Gamma_D(v_\tvu), \dots$ for $D$ given in Example \ref{exp: s adm}. However, Theorem \ref{thm: each w in less than a Sadm}, indicates that any strongly admissible interpretation of ADF $D$ is bounded by an element of the sequence of strongly admissible interpretations  presented in Lemma \ref{lem: constructing Sadm}.

\begin{theorem}\label{thm: each w in less than a Sadm}
Let $D$ be an ADF, let $w$ be an interpretation of $D$, and let $v_i$ for $0\leq i$ be the sequence of interpretations presented in Lemma \ref{lem: constructing Sadm}.  If $w$ is a strongly admissible interpretation of $D$, then there exists the least $0\leq m$ such that $ w \leq_i v_{m}$.  
\end{theorem}
\begin{proof}
Let $\{a_1, \ldots, a_n\}$ be the set of arguments the truth values of which appear in $w$. Further, assume that each $a_i$ is strongly acceptable/deniable with respect to $w$ and a least  set $S_i$. Let $S= \bigcup_{i=1}^n S_i$. Let $a$ be an argument with the greatest maximum level $m$ in $S$. 
We claim that $w \leq_i v_m$.  
We have to show that if $a_i\mapsto\tvt/\tvf\in w$, then $a_i\mapsto\tvt/\tvf\in v_m$.  We show our claim 
by  induction on the maximum level of arguments in $S$.

Base case: If $a\mapsto\tvt/\tvf\in w$ and the maximum level of $a$ in $S$ is $1$, then it is clear that $w\leq_i \Gamma(v_0)=v_1$. Therefore,  $a\mapsto\tvt/\tvf\in v_m$.

As induction hypothesis,  assume that if $a\mapsto\tvt/\tvf\in w$ and the maximum level of $a$ in $S$ is $j$ with $1\leq j\leq k<m$, then $a\mapsto\tvt/\tvf\in v_j$ (and also $a\mapsto\tvt/\tvf\in v_m$).

Induction step:  Assume that $a$ is an argument that is strongly acceptable/deniable with respect to $w$ and the maximum level $a$ in $S$ is $k+1$. We have to show that $a\mapsto\tvt/\tvf\in v_{k+1}$ (and $a\mapsto\tvt/\tvf\in v_m$).  Since $a$ is strongly acceptable/deniable with respect to  $w$ and $S$ and the maximum level of $a$ in $S$ is $k+1$, there exists a non-empty set $P\subseteq\parents a$ such that $\varphi_a^{w_{|_P}}\equiv\top/\bot$. 
Since $p$ is a parents of $a$, by Definition \ref{def: a sadm w.r.t v A}, $p$ is also strongly acceptable/deniable with respect to  $w$ and $S$. Thus, by Definition \ref{def: level based on least S} the maximum level of each $p$ is strictly less than the maximum level of $a$ i.e.\ the maximum level of $p$ in $S$ is at most $k$. Then, by the induction hypothesis, $p\mapsto\tvt/\tvf\in v_k$, for each $p\in P$. 
Therefore, $\varphi_a^{w_{|_P}}\equiv\varphi_a^{{v_k}_{|_P}}$. Further, $\varphi_a^{{v_k}_{|_P}}\equiv\varphi_a^{v_{|_P}}$ because ${v_k}_{|_P}\leq_i v_k$ and $\Gamma_D$ is a monotonic function. Therefore, $a\mapsto\tvt/\tvf\in v_{k+1}$ (and also $a\mapsto\tvt/\tvf\in v_m$). That is, there exists an $m \geq 0$, such that $w\leq_i v_m$.

Further, we have to show that the  natural number $m$ assumed  in the beginning of the proof is the least natural number that satisfies the condition of the theorem. 
Toward a contradiction assume that there exists an $m'< m$ such that  $w\leq_i v_{m'}$. By our assumption the greatest maximum level of an argument of $w$, namely $a$ is $m$ and $S$ is a least set that satisfies the conditions of Definition \ref{def: a sadm w.r.t v A} for all arguments the truth values of them appear in $w$. It is easy to check that $\Gamma_D^{m'}v_{0}(a)=\tvu$. Thus, $w\not\leq_i v_{m'}$. That is, $m$ is the least natural number that satisfies the condition of the theorem.   

\end{proof}
\begin{theorem}
Let $D$ be an ADF and let $v$ be an interpretation of 
$D$. If argument $a$ is strongly acceptable/deniable with respect to $v$ and a least  set $S$, then each $s\in S$ is also strongly acceptable/deniable with respect to $v$ and a $S'\subseteq S$. 
\end{theorem}
\begin{proof}
Toward a contradiction assume that there exists $s\in S$ that is not strongly acceptable/deniable with respect to $v$ and any $S'\subseteq S$. By Definition \ref{def: a sadm w.r.t v A}, any argument in set $E\setminus\{a\}$ is an ancestor of $a$ that is strongly acceptable/deniable. Thus,  $s$ is not any of the ancestors of $a$ that appears in set $E$, otherwise it is strongly acceptable/deniable. Therefore, $a$ is also strongly acceptable/deniable with respect to $v$ and $S\setminus\{s\}$. Then, $S$ is not a least set that satisfies the conditions of Definition \ref{def: a sadm w.r.t v A} for $a$. This is a contradiction by the assumption of the theorem that $S$ is a least set. Thus, the assumption that there exists an argument in $S $ that is not strongly acceptable/deniable with respect to  $v$ and a subset of $S$ is wrong. 
\end{proof}
\noindent	To show that the set of strongly admissible interpretations of a given ADF make a lattice, first, in Theorem  \ref{lemma: sup of sadm} we show that every two strongly admissible interpretations of $D$ have a unique supremum. To this end, we first introduce the notion of \emph{join} of two strongly admissible interpretations in Definition \ref{def: join}.

\begin{definition}\label{def: join}
Let $D$ be an ADF and let $v$ and $w$ be two strongly admissible interpretations of $D$. The \emph{join} $v\sqcup_i w$ is defined as
\[
v\sqcup_i w(a)= \begin{cases}
	v(a) & \quad \text{if there exists $a$ s.t. }a\mapsto\tvt/\tvf\in v,\\
	w(a) & \quad \text{if there exists $a$ s.t. }a\mapsto\tvt/\tvf\in w,\\
	\tvu & \quad \text{otherwise.}
\end{cases}
\]
\end{definition}
\begin{proposition}
The join of two strongly admissible interpretations of $D$ is a well-defined function. 
\end{proposition}
\begin{proof}

Let $D$ be an ADF and let $v$ and $w$ be two strongly admissible interpretations of $D$.
We show that the join operator is a well-defined function. That is, we have to show that there is no $a$ that has two different values in $v\sqcup_i w$. Toward a contradiction assume that there is a $a$ that has two different outputs in $v\sqcup_i w$. That is, $a$ is assigned to $\tvt$ in one of the interpretations and to $\tvf$ in another one. For instance, $v(a)=\tvt$ and $w(a)=\tvf$. By Theorem \ref{thm: each w in less than a Sadm},   there exists the least natural numbers  $k$ and $m$ such that $v\leq_i v_k$ and $w\leq_i v_m$, respectively. Since  $v\leq_i v_k$ and $v(a)=\tvt$, $a\mapsto\tvt\in v_k$. Further, since $w\leq_i v_m$ and $w(a)=\tvf$, $a\mapsto\tvf\in v_m$. That is, $v_k\not\leq_i v_m$ and $v_m\not\leq_i v_k$. This is a contradiction  by Lemma \ref{lem: constructing Sadm},  that says either $v_k\leq_i v_m$ or $v_m\leq_i v_k$, because $v_k$ and $v_m$ are elements of the sequence of interpretations presented in Lemma \ref{lem: constructing Sadm}. Thus, the assumption that there exists $a$ that is acceptable in a strongly admissible interpretation of $D$ but that is deniable in another strongly admissible of $D$ is wrong. Thus,  $v\sqcup_i w$ is a well-defined function.
\end{proof}
\noindent
Lemma \ref{lem: join is sadm}, presents that the join of two strongly admissible interpretations of a given ADF is also a strongly admissible interpretation of that ADF.
\begin{lemma}\label{lem: join is sadm}
Let $D$ be an ADF and let $v$ and $w$ be strongly admissible interpretations of $D$. Then $v\sqcup_i w $  is also a strongly admissible interpretation of $D$.
\end{lemma}
\begin{proof}
Toward a contradiction assume that $v\sqcup_i w $ is not a strongly admissible interpretation of $D$. Thus, there exists an $a$ such that  $a\mapsto\tvt/\tvf\in v\sqcup_i w$ but it is not strongly acceptable/deniable with respect to $v\sqcup_i w $ and any  set. By Definition \ref{def: join}, either $a\mapsto\tvt/\tvf \in v$ or $a\mapsto\tvt/\tvf \in w$. Since $v$ and $w$ are strongly admissible interpretations, $a$ is strongly acceptable/deniable with respect to $v$ or $w$.  Since $v\leq_i v\sqcup_i w$ and $w\leq_i v\sqcup_i w$, by Lemma \ref{lem: sadm in greater int}, 
$a$ is strongly acceptable/deniable with respect to $v\sqcup_i w$. This is a contradiction with the assumption that $a$ is not strongly acceptable/deniable with respect to $v\sqcup_i w$. Thus, the assumption that $v\sqcup_i w$ is not a strongly admissible interpretation was wrong. That is, the join of two strongly admissible interpretations of $D$ is a strongly admissible interpretation of $D$.
\nop{
	(bi rabt)
	Let $v_i$ for $i\geq 0$ be the sequence of  interpretations presented in Lemma \ref{lem: constructing Sadm}. 
	By Theorem \ref{thm: each w in less than a Sadm}, there exists least natural numbers of $m$ and $k$ such that $v\leq_i v_m $ and $w\leq_i v_k$.  Assume that $m\geq k$. Thus,  $w\leq_i v_m$. Therefore, every two strong admissible interpretation has an upper bound.} 
\end{proof}

\begin{theorem}\label{lemma: sup of sadm}
Let $D$ be an ADF.  Every two strongly admissible interpretations of $D$ have a unique supremum. 

\end{theorem}
\begin{proof}
Let $D$ be an ADF and let $v$ and $w$ be two strongly admissible interpretations of $D$. We show that $v\sqcup_i w$ is a supremum of $v$ and $w$.
By Definition \ref{def: join}, $v\sqcup_iw$ is an upper bound of $v$ and $w$. By Lemma \ref{lem: join is sadm}, $v\sqcup_i w$ is a strongly admissible interpretation of $D$. It remains to show that $v\sqcup_i w$ is a least upper bound of $v$ and $w$. Toward a contradiction, assume that $v\sqcup_i w$ is not the least upper bound of $v$ and $w$. That is, there exists  a   strongly admissible interpretation $w'$ of $D$ such that  $v\leq_i w'$, $w\leq_i w'$ and $w'<_i v\sqcup_i w$. Thus there exists $a$ with $a\mapsto\tvu\in w'$ and $a\mapsto\tvt/\tvf\in v\sqcup_i w$. Thus, either $a\mapsto\tvt/\tvf\in v$ or $a\mapsto\tvt/\tvf\in w$. That is, either $v\not\leq_i w'$ or $w\not\leq_i w'$. This is a contradiction by the assumption that $w'$ is the least upper bound of $v$ and $w$. Thus, the assumption that $v\sqcup_i w$ is not the least upper bound of $v$ and $w$ was wrong.  
\end{proof}
\noindent Further, to show that the set of strongly admissible interpretations of ADF $D$ make a lattice, in Theorem  \ref{lem: inf of sadm} we show that every two strongly admissible interpretations of $D$ have an infimum. To this end, in Definition \ref{def: uniq max sadm}, we present the concept of the maximum strongly admissible interpretation contained in an interpretation of $D$.  
\begin{definition}\label{def: uniq max sadm}
Let $D$ be an ADF and let $v$ be an interpretation of $D$. Interpretation $w$ is called a unique maximum strongly admissible interpretation that is less than or equal to $v$,  with respect to $\leq_i$ ordering if the following conditions hold:
\begin{itemize}
	\item $w$ is a strongly admissible interpretation of $D$ s.t. $w\leq_i v$,
	\item  there is no  strongly admissible interpretation $w'$ of $D$ such that  $w<_i w'\leq_i v$.
	
\end{itemize}
\end{definition}
\begin{lemma}\label{lem: every int has sadm}
Let $D$ be an ADF and let $v$ be an interpretation of $D$. Then, there exists  a unique maximum strongly admissible interpretation that is less than or equal to $v$,  with respect to $\leq_i$ ordering. 
\end{lemma}
\begin{proof}
Each interpretation of $D$ has at least as much information as   the trivial interpretation. Thus, each $v$ of $D$ has at least as much information as  $v_\tvu$, which is a strongly admissible interpretation. Since the number of arguments of $D$ is finite, there exists at least one maximal strongly admissible interpretation of $D$, 
namely $w$ for the given interpretation $v$. We show that this $w$ is unique. Toward a contradiction assume that there are two maximal strongly admissible interpretations that satisfy the condition of the lemma, namely $w$ and $w'$. By Lemma \ref{lem: join is sadm}, $w\sqcup_i w'$ is a strongly admissible interpretation of $D$ s.t. $w\sqcup_i w'\leq_i v$. However, $w\leq_i w\sqcup_i w'$ and $w'\leq_i w\sqcup_i w'$ together with  the assumption that  $w$ and $w'$ are maximal strongly admissible interpretations lead to   $w\sim_i w\sqcup_i w'$ and $w' \sim_i w\sqcup_i w'$. That is, $w \sim_i w'$.     Thus, the maximum strongly admissible interpretation which is contained in $v$ is unique. 
\end{proof}
\begin{theorem}\label{lem: inf of sadm}
Let $D$ be an ADF. Every two strongly admissible interpretations of $D$ have a unique infimum.
\end{theorem}
\begin{proof}
Let $D$ be an ADF and let $v$ and $v'$ be two strongly admissible interpretations of $D$. Let $w = v\sqcap_i v'$. By Lemma \ref{lem: every int has sadm},  there exists a unique maximum strongly admissible interpretation $w'$ that is less than or equal to $w$, i.e.\ $w'\leq_i w$. That is $w'$ is a lower bound of $v$ and $v'$. It remains to show that $w'$ is the greatest lower bound of $v$ and $v'$. Toward a  contradiction assume that there exists $w''$ that is the greatest lower bound of  $v$ and $v'$. That is, $w''\leq_i v$ and $w''\leq_i v'$. Then by the definition $w''\leq_i  (v\sqcap_i v'= w)$. By the assumption $w'$ is the maximum strong admissible that is less or equal to $w$, thus, $w''\leq_i w'$. Thus, $w'$ is an infimum of $v$ and $v'$.
\end{proof}

\begin{theorem}\label{thm: sadf form a lattice}
Let $D$ be an ADF. The strongly admissible interpretations of $D$ form a lattice with respect to the $\leq_i$-ordering, with the least  element $v_\tvu$ and the top element $\grd(D)$. 
\end{theorem}
\begin{proof}
We have to show that every two strongly admissible interpretations of $D$ have a supremum and an infimum. Theorem \ref{lemma: sup of sadm} shows the former one and Theorem \ref{lem: inf of sadm} indicates the latter one. Thus, the strongly admissible interpretations of $D$ make a lattice with respect to the $\leq_i$-ordering. In Lemma \ref{prop: big and smal sadm}, it is shown that $v_\tvu$ is the least strongly admissible interpretation and $\grd(D)$ is the largest strongly admissible interpretation of the sequence of the interpretations presented in Lemma \ref{lem: constructing Sadm}. This fact together with Theorem \ref{thm: each w in less than a Sadm}, shows that $\grd(D)$ is the greatest element of this lattice. It is trivial that $v_{\tvu}$ is the least element of this lattice. 
\end{proof}
\noindent	The set of strongly admissible interpretations of ADF $D= (\{a,b,c,d\},\newline \{\varphi_a: \top, \varphi_b: a\land \neg c, \varphi_c: \neg b \land d, \varphi_d: \bot \})$, given in Example \ref{exp: s adm} form a  lattice, depicted in Figure \ref{fig: lattice}. The top element of this lattice is $\grd(D)= \{a\mapsto\tvt, b\mapsto\tvt, c\mapsto\tvf, d\mapsto\tvf\}= \{a, b, \neg c, \neg d\}$. 

\begin{figure}[t]

\centering
\begin {tikzpicture}[black!70, >=stealth,node distance=0.6cm,
thick,main node/.style={fill=none,draw,minimum size = 0.4cm,font=\normalsize\bfseries},
condition/.style={fill=none,draw=none,font=\small\bfseries}, scale=0.6]
\path
(0,0)node (a) {$\{\}$}
++(-2.,2) node (b) {$\{a\}$}
++(4,0) node (c) {$\{\neg d\}$}
++(-2,2) node (d) {$\{a, \neg d\}$}
++(4,0) node  (e) {$\{\neg c, \neg d\}$}
++(-2,2) node (f) {$\{a, \neg d, \neg c\}$}
++(-2,2) node (g) {$\{a, b, \neg c, \neg d\}$};

\path[thick, -] 
(a) edge  (b)
(a) edge  (c)
(b) edge  (d)
(c) edge  (d)
(a) edge  (d)
(c) edge  (e)
(e) edge  (f)
(d) edge  (f)
(f) edge  (g)
(d) edge  (g)
(b) edge  (g);

\end{tikzpicture}
\caption{Complete lattice of the strongly admissible interpretations of the  ADF of Example~\ref{exp: s adm}}
\label{fig: lattice}
\end{figure}
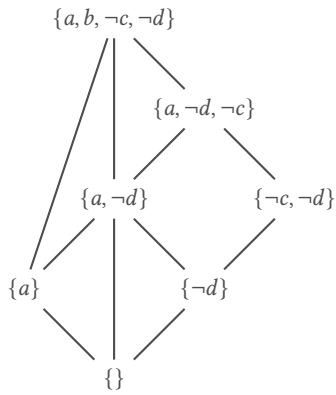
\section{Conclusion}\label{sec: con}
In this work we have defined strongly admissible semantics for ADFs, based on the concept of strongly acceptable/deniable arguments.
From a theoretical perspective, we have observed
that the strongly admissible interpretations of a given ADF form a lattice with the trivial interpretation as the unique minimal element and
the grounded interpretation as the unique maximal element.

The concept of strong admissibility is related to grounded semantics in a similar
way as the concept of admissibility is related to preferred semantics. That is, to answer the credulous decision problem  of an ADF under the grounded semantics, there is no need of constructing the full grounded interpretation of the given ADF. Instead, it is enough to construct a strongly admissible interpretation of the given ADF that satisfies the decision problem.  Similarly, to answer the credulous decision problem of ADFs under preferred semantics, it is enough to investigate whether there exists an admissible interpretation in order to solve the decision problem.We used this method in  preferred  discussion games in \cite{zafarghandi2019discussion} to answer the credulous decision problem of ADFs under preferred semantics.

Possible future research questions include whether the concept of strongly admissible semantics for ADFs, presented in this work,  is a proper generalization of the concept of strongly admissible semantics for AFs \cite{DBLP:journals/ai/BaroniG07,DBLP:conf/comma/Caminada14}.

Further, it is interesting to investigate how the concept of strong admissibility of ADFs relates to the grounded discussion game presented in \cite{DBLP:conf/comma/KeshavarziVV20}. In other words,   investigation is required of the question  whether the discussion game presented in \cite{DBLP:conf/comma/KeshavarziVV20} to answer the credulous decision problem of ADFs under the grounded semantics  is equivalent to answer  the same decision problem under strong admissibility interpretation. The grounded discussion game was defined over ADFs without any redundant links, however, the concept of strongly admissible semantics is presented for all kinds of ADFs. Thus, we will investigate whether the concept of strongly admissible semantics is at the
basis of  the proof procedures of the grounded discussion games for ADFs without any redundant links.

Further, we would like to investigate whether   the grounded discussion game presents the shortest discussion/explanation that answers the credulous  decision problems under strongly admissible/
grounded semantics  for the given argument of ADFs.

Computational complexity classes of  semantics of AFs and ADFs are presented in \cite{Dvork2017ComputationalPI}.
Computational complexity of strongly admissible semantics of AFs is studied in \cite{DBLP:conf/comma/DvorakW20}.
Further, in \cite{DBLP:conf/comma/CaminadaD20}, the computational complexity of
identifying strongly admissible labellings with bounded or minimal size was studied. As a future work, it would be interesting to clarify the computational complexity of investigating of the truth value of an argument in a strongly admissible interpretation of a given ADF.

\begin{acks}
The authors would like to thank Dr. M. Caminada and Prof. dr. S. Woltran for their recommendations for presenting the notion of strongly admissible semantics for ADFs. The work is supported by 
the Center of Data Science $\&$ Systems Complexity (DSSC) Doctoral Programme, at the University of Groningen.
\end{acks}

\bibliographystyle{ieeetr}
\bibliography{StrongAdmissibilityofADFs} 
\end{document}